\RequirePackage{fix-cm}
\documentclass[smallextended]{svjour3}       
\smartqed  
\usepackage{graphicx}

\usepackage{natbib}
\usepackage[utf8]{inputenc} 
\usepackage[T1]{fontenc}    
\usepackage{hyperref}       
\usepackage{url}            
\usepackage{booktabs}       
\usepackage{amsfonts}       
\usepackage{nicefrac}       
\usepackage{microtype}      

\usepackage{pdfpages}
\usepackage{tikz}
\usetikzlibrary{plotmarks,arrows,calc}
\usepackage{pgfplots}
\usepackage{listings}
\usepackage[utf8]{inputenc}
\usepackage[T1]{fontenc}
\usepackage{fixltx2e}
\usepackage{graphicx}
\usepackage{longtable}
\usepackage{float}
\usepackage{wrapfig}
\usepackage{soul}
\usepackage{hyperref}
\usepackage[normalem]{ulem}
\usepackage{booktabs}
\usepackage{psfrag}
\usepackage[b]{esvect}
\usepackage{framed,color}
\definecolor{shadecolor}{rgb}{0.6,0.3,0}
\usepackage{tcolorbox}
\usepackage[vertfit]{breakurl}
\usepackage{filecontents}
\usepackage{relsize}
\usepackage{caption}
\usepackage{bbm}

\usepackage{latexsym}
\usepackage{amsopn}
\usepackage{amssymb}
\usepackage{subfig}
\usepackage{algorithm}
\usepackage[noend]{algpseudocode}
\algnewcommand\algorithmicinput{\textbf{Input:}}
\algnewcommand\INPUT{\item[\algorithmicinput]}
\algnewcommand\algorithmicoutput{\textbf{Output:}}
\algnewcommand\OUTPUT{\item[\algorithmicoutput]}

\usepackage{amsmath}
\usepackage{lineno}
\usepackage{hyperref}
\usepackage{paralist}

\newcommand{\beq}{\begin{eqnarray*}}
\newcommand{\eeq}{\end{eqnarray*}}
\newcommand{\beqn}{\begin{eqnarray}}
\newcommand{\eeqn}{\end{eqnarray}}

\newcommand{\field}[1]{\mathbb{#1}}
\newcommand{\reals}{\field{R}}

\newcommand{\sign}{{\mbox{sign}}}
\newcommand{\Y}{\mathcal{Y}}
\newcommand{\Z}{\mathcal{Z}}
\newcommand{\E}{\mathop{\mathbb{E}}}
\newcommand{\F}{\mathcal{F}}
\newcommand{\X}{\mathcal{X}}
\renewcommand{\H}{\mathcal{H}}
\newcommand{\fix}{^{\textrm{{\tiny \textup{FIX}}}}}
\renewcommand{\P}{\mathop{\mathbb{P}}}
\newcommand{\paren}[1]{\left( #1 \right)}
\newcommand{\sqprn}[1]{\left[ #1 \right]}
\newcommand{\gn}{\, | \,}
\newcommand{\set}[1]{\left\{ #1 \right\}}
\newcommand{\chr}{\boldsymbol{\mathbbm{1}}} 
\newcommand{\pred}[1]{\chr_{\left\{ #1 \right\}}}
\newcommand{\err}{\operatorname{err}}
\newcommand{\herr}{\widehat{\operatorname{err}}}
\newcommand{\inv}{^{-1}} 
\newcommand{\sgn}{\operatorname{sgn}}
\newcommand{\nrm}[1]{\left\Vert #1 \right\Vert}

\newcommand{\mycomment}[1]{}

\usepackage{todonotes}

\newcommand{\N}{\mathbb N}
\newcommand{\R}{\mathbb R}
\newcommand{\eps}{\varepsilon}
\DeclareMathOperator{\vol}{vol}
\DeclareMathOperator{\surf}{surf}

\begin{document}

\title{Nested Barycentric Coordinate System as an Explicit Feature Map}


\author{Lee-Ad Gottlieb         \and
        Eran Kaufman \and
        Aryeh Kontorovich \and
        Gabriel Nivasch \and
        Ofir Pele
}


\institute{            Lee-Ad Gottlieb  \at
              Ariel University \\
              \email{leead@ariel.ac.il}                 
                  \and
Eran Kaufman \at
              Ariel University \\
              \email{erankfmn@gmail.com}           
           \and
           Aryeh Kontorovich
	  \at
              Ben-Gurion University  \\
              \email{karyeh@cs.bgu.ac.il}   
                 \and
               Gabriel Nivasch
	  \at
               Ariel University   \\
              \email{gnivasch@yahoo.com}   
                \and
               Ofir Pele
	  \at
               SanDisk   \\
              \email{ofirpele@gmail.com}   
}

\date{Received: date / Accepted: date}

\maketitle

\begin{abstract}
We propose a new embedding method which is particularly well-suited for settings where the sample size greatly exceeds the
ambient dimension. Our technique consists of partitioning the space into simplices and then embedding the data points into
features corresponding to the simplices' barycentric coordinates. We then train a linear classifier in the rich feature
space obtained from the simplices. The decision boundary may be highly non-linear, though it is linear within each simplex
(and hence piecewise-linear overall). Further, our method can approximate any convex body.
We give generalization bounds based on empirical margin and a novel hybrid sample compression technique.
An extensive empirical evaluation shows that our method consistently outperforms
a range of popular kernel embedding methods.
\end{abstract}

%



\section{Introduction}

Kernel methods provide two principal benefits:
(1) They implicitly induce a non-linear feature map, which allows for a richer space of classifiers
and 
(2) when the kernel trick is available, they effectively replace the dimension $d$ of the feature
space with the sample size $n$ as the computational complexity parameter.
As such,
these are well-suited for the `high dimension, moderate data size' regime.
For very large datasets, however, naive use of kernel methods becomes prohibitive.
The cost is incurred both at the training stage, where an optimal classifier is searched for
over an $n$-dimensional space, and at the hypothesis evaluation stage, where a sum of $n$ kernel
evaluations must be computed.

For these reasons, for large data sets, {\em explicit feature maps} are sometimes preferred.
Various approximations have been proposed to mitigate the computational challenges
associated with explicit feature maps, including
\citet{chang2010training,maji2012max,perronnin2010large,rahimi2007random,vedaldi12efficient,Li2010,ST-18,Chum-15,ZK-13}.

\paragraph{Our contribution.}
We propose a new embedding method which is well-suited for the large sample regime.
Our technique consists of partitioning the space into a nested hierarchy of simplices,
and then embedding each data point into features corresponding to the barycentric coordinates
of the simplex that contains it.
We then train a linear classifier in the rich feature space obtained from the simplices.
For sample size $n$ in $d$-dimensional space, our algorithm has runtime
$O(d^2n)$ regardless of the dimension of the embedding space 
(when the approximation parameter is taken to be fixed, see Sections \ref{sec:learning}
and \ref{sec:exp}). 
In contrast, standard kernelized SVM has a runtime $O(d n^2)$.

Additionally, our embedding technique allows for highly non-linear decision boundaries,
although these are linear within each simplex (and hence piecewise-linear overall), as
explained in Section \ref{sec:emb}.
At the same time, our approach is sufficiently robust to closely approximate realizable
convex bodies -- in fact, multiple such bodies -- in only linear time in fixed dimension
(Section \ref{sec:approx}).
We also give generalization bounds based on empirical margin (Theorem~\ref{thm:bounds})
and a novel hybrid sample compression technique (Theorem~\ref{thm:adapt}).
%
%
%
Finally, we perform an extensive empirical evaluation, in which
our method consistently outperforms other explicit feature map
classification methods, including a range of popular kernel embedding methods 
(Section \ref{sec:exp}).

\subsection{Related Work} \label{sec:Related}
Kernel approximations for explicit feature maps come in two basic varieties:
data-independent approximations to fixed kernels, and data-dependent feature maps.

\paragraph{Data-dependent kernel approximations.}
This category includes
Nystrom's approximation \citep{Williams00theeffect}, which projects the data onto a
suitably selected subspace. If $K(x, z_i)$ is the
projection of example $\vec x$ onto the basis element $\vec z_i$,
the points
$\{\vec z_{1}, \ldots , \vec z_{n}\}$ are chosen to
maximally capture the data variability. Some methods select $\vec z_i$ from the 
sample. The selection can be random
\citep{NIPS2000_1866}, greedy
\citep{Smola:2000:SGM:645529.657980}, or involve an incomplete Cholesky
decomposition \citep{DBLP:journals/jmlr/FineS01}.
\citet{perronnin2010large} applied Nystrom's
approximation to each dimension of the data independently, greatly
increasing the efficiency of the method.

\paragraph{Data-independent kernel approximations.}
This category includes sampling the Fourier domain to compute explicit maps for
translation invariant kernels.
\citet{rahimi2007random,NIPS2008_3495} 
do this for the radial basis function kernel, also known as Random Kitchen Sinks.
\citet{Li2010,vedaldi12efficient}
applied this technique to certain group-invariant
kernels, and proposed an adaptive approximation to the $\chi^2$ kernel.
\citet{6027289} map the input data onto a low-dimensional spectral (Fourier) feature
space via a cosine transform.
\citet{DBLP:conf/bmvc/VempatiVZJ10} proposed a skewed chi squared kernel,
which allows for a simple Monte Carlo approximation of the feature map.
\citet{maji2012max} approximated the intersection kernel and the $\chi^2$ kernel by
a sparse closed-form feature map.
\citet{Pele-icml2013} suggested using not only piecewise linear function in each feature separately
but also to add all pairs of features.
\cite{chang2010training} conducted an extensive study on the usage of
the second-order polynomial explicit feature map.
\cite{bernal2012automated} approximated second order features relationships
via a Conditional Random Field model.

\paragraph{Decompositions and other SVM approaches.}
Simplex decompositions have been used to produce proximity-based classifiers \citep{BHP-18, Davies-96},
but to the best of our knowledge, ours is the first work to utilize either 
nested simplex decompositions or barycentric centers in conjunction with SVM. 
Simplex decompositions are related to the quadtree, and the quadtree has
been used together with SVM for various learning tasks \citep{SGM-04,BD-15}, 
but not for the creation of a kernel embeddings.
Simplex decompositions are more efficient than quadtrees,
since a simplex naturally decomposes into only $d+1$ sub-simplices (Section \ref{sec:emb}),
while a quadtree cell naturally decomposes into $2^d$ sub-cells.

As mentioned, our emphasis in this paper is specifically on explicit feature maps,
but there are numerous approaches to reducing kernel SVM runtime 
(for example the CoreSVM of \citet{TKC-05, TKK-07}).
Another related paradigm is that of {\em Local SVM} \citep{ZBMM-06, GH-13}, 
which assumes continuity of the labels with respect to spacial proximity;
similarly labeled points tend to cluster together.
This differs from the underlying assumption motivating kernel SVM,  
which assumes that the data is approximately linearly separable, 
but not necessarily clusterable.
These approaches find success in distinct settings, and are incomparable.

\paragraph{Approximating convex polytopes.}
Learning arbitrary convex bodies requires very large sample size \citep{GL09},
and so we focus instead on convex polytopes defined by a small number
of halfspaces. However, the problem of finding consistent polytopes is 
known to be $\mathrm{NP}$-complete even when the polytope 
is simply the intersection of two hyperplanes \citep{Megiddo1988}.
In fact, \citet{DBLP:journals/jcss/KhotS11}
showed that ``unless $\mathrm{NP}=\mathrm{RP}$, it is hard to (even) weakly
PAC-learn intersection of two halfspaces'',
even when allowed the richer class of $O(1)$ intersecting halfspaces.
\citet{DBLP:journals/jcss/KlivansS09} showed that 
learning an intersection of $n^\eps$
halfspaces is intractable regardless of hypothesis representation 
(under certain cryptographic assumptions).
These negative results have motivated researchers to consider the problem of 
discovering consistent polytopes which have some separating margin. 
Several approximation and learning algorithms have been suggested for
this problem, featuring bounds with steep dependence on the inverse margin
and number of halfspaces forming the polytope
\citep{DBLP:journals/ml/ArriagaV06, DBLP:journals/jcss/KlivansS08,DBLP:conf/nips/GottliebKKN18, GK18}. 

In contrast, we show in Section \ref{sec:approx} that our method is capable of 
approximating any convex polytope in linear time
(in fixed dimension), independent of the halfspace number 
and with only logarithmic dependence on the inverse margin.
It accomplishes this by finding a linear separator in the higher-dimensional embedded space,
and projecting the solution back into the origin space. 
However, our approach is not strictly comparable to those above, 
as they are concerned with minimizing the disagreement 
between the computed polytope (or object) and the true underlying polytope
with respect to the point space,
while we minimize the volume of the space between them.


\section{The barycentric coordinate system}\label{sec:emb}
Here we describe the nested barycentric coordinate system embedding. 
We explain its construction and description, 
how to embed a point from the origin space into the new coordinate
system, and how a point in the embedded system can be projected 
back into the origin space (Section \ref{sec:bary}). 
We then show that if we associate a weight with each simplex point,
then the embedding and weights together imply some (not necessarily convex) 
polytope on the origin space (Section \ref{sec:weight}). 
Later in Section \ref{sec:approx}, we will show that this system
is sufficiently robust that it can be used to approximate any convex
body.

\subsection{Nested barycentric embedding}\label{sec:bary}

Let $S\subset \R^d$ be a regular simplex of unit side-length, and let $\{q_0, \ldots, q_d\}$ be its vertices.
 Each  point $\vec x$ inside the simplex can be written using the barycentric coefficients:
\begin{equation}
\label{eq:proj}
\begin{aligned}
	 &\vv x  = \sum_{i=0}^{d} \alpha_{i}(\vv x) \vv q_i\\
  &\sum_{i=0}^{d} \alpha_{i}(\vv x) = 1 \quad  0 \leq \alpha_{i} \leq 1
\end{aligned}
\end{equation}
 Here $\alpha_i(\vv x)$  denotes the coefficient of point $\vv x$ corresponding to vertex $q_i$.
 Let the ordered vector of $\alpha$'s, $\{\alpha_0 , \ldots , \alpha_d \}$, 
corresponding to $\vv x$  be denoted as  $\phi_{d+1}(\vv x)$.
If we artificially augment the original feature space by adding another feature which equals to 1, i.e. $\vec x = (x_0 ,\ldots , x_{d-1},1)$
and $\vec q_t = (q_{1t} , \ldots , q_{d-1,t},1)$ and define the matrix $Q_t := (\vec q_0 , \dots , \vec q_t)$, then the transformation 
$\phi_{d+1}(\vv x) : \reals^{d+1} \rightarrow \reals^{d+1}$ is a linear transformation of the form $\vec x = Q \vec  \alpha$. 

We can further refine the system by 
introducing a new point $q_{d}$ inside the simplex, 
thereby splitting the simplex into $d+1$ new sub-simplices.
We order the coordinates of our system as $\{q_0, \ldots , q_{d} \}$.
A point $\vv x$ inside the system is embedded by 
first utilizing the $d+1$ vertices of its surrounding simplex
to compute the barycentric coefficients (the $\alpha$'s) of equation \ref{eq:proj}.
Then $\vv x$ is assigned a vector wherein a coordinate corresponding 
to one of these $d+1$ simplex vertices is set to the coefficient of that vertex,
and all other coordinates are set equal to 0.
This defines the embedding $\phi_{d+2}(\vv x) : \reals^{d+1} \rightarrow \reals^{d+2} $.

The refinement process can be continued by choosing points inside simplices to further split these simplices.
We define the nested architecture $B_t$ and its associated embedding $\phi_t(\vec x)$  
to be the coordinates $\{\vec q_0, \ldots , \vec q_{t} \}$ constructed 
from $B_{t-1}$ by concatenating a new point $\vec q_t$ at step $t$ to the previous coordinate system.
Each point $\vv x$ is embedded using the barycentric coefficients of the vertices of the simplex surrounding 
point $\vec x$ and by assigning those coefficients in the index of the corresponding vertices and by assigning zero to all other vertices.
We note that the embedding --- the {\em nested barycentric coordinate system} ---
is sparse, as only $d+1$ coefficients are non-zero,
and also that the embedded points lie on the $L_{1}$ sphere ($\sum \alpha = 1$).


\begin{figure}[htbp]
\centering
\subfloat[step 1]{\label{fig:a}\includegraphics[width=0.35\linewidth]{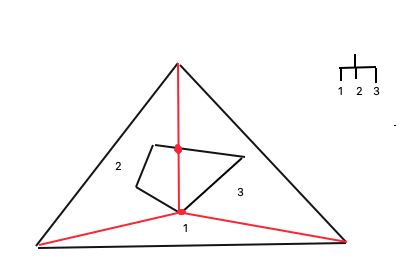}}
\subfloat[step 2.1]{\label{fig:b}\includegraphics[width=0.35\linewidth]{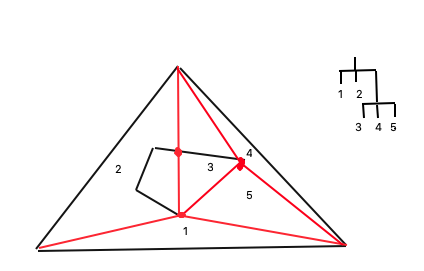}}
\subfloat[step 2.2]{\label{fig:c}\includegraphics[width=0.35\linewidth]{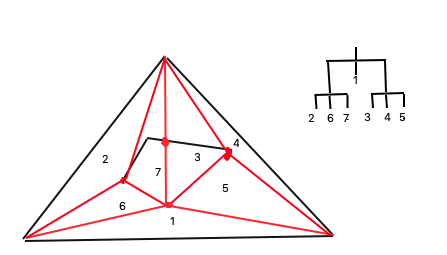}}\\
\subfloat[step 3]{\label{fig:d}\includegraphics[width=0.35\linewidth]{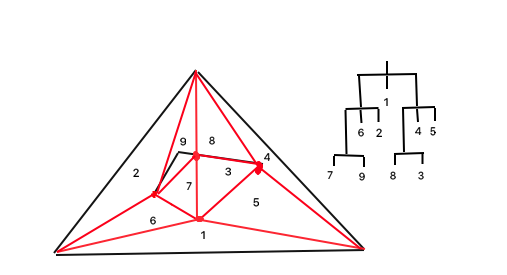}}
\subfloat[step 4]{\label{fig:e}\includegraphics[width=0.35\linewidth]{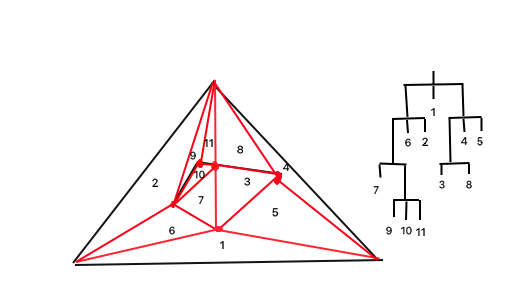}}
\caption{Creation of the nested system about a convex polytope}
\label{fig:split_proc}
\end{figure}
A point in the embedded space can be projected back into the origin space 
by utilizing the identity
 \begin{equation}
\label{eq:proj2}
\begin{aligned}
	 &\vv x  = \sum_{i=0}^{t} \alpha_{i}(\vv x) \vv q_i\\
\end{aligned}
\end{equation}

\subsection{Weights, hyperplanes and polytopes}\label{sec:weight}
Given an embedding, we will assign a set of weights 
$\vv w$ to the vertices $\{\vec q_0, \ldots, \vec q_t\}$.
Then the set of points $R$ such that: 
\begin{equation}
R= \{\vec x\in S : \vec w\cdot \phi_t(\vec x) \ge 0\}
\end{equation}
is a union of regions whose boundaries are unions and intersections of hyperplanes.
$R$ can represent a polytope as well as the union of several disjoint polytopes, 
each of which is not necessarily convex
(see Figure \ref{fig:split} for an illustration):

\begin{lemma}
Any hyperplane that crosses a single simplex can be defined by a set of weights 
$\vv w = \{w_0 , \ldots , w_d\}$,
such that all points $\vv x$ that lie on the hyperplane satisfy the equation $\vv w \cdot  \phi_{d+1}(\vv x) = 0 $.
Further, $R$ is a union of regions whose boundaries are unions and intersections of hyperplanes, 
where each simplex contains at most one hyperplane.
\end{lemma}

\begin{proof}
Choose $d$ linearly independent points on the given hyperplane. 
Since these points are inside the coordinate system, 
they have unique barycentric coefficients and thus a unique representation.
Finding these weights is equivalent to solving
$A  \cdot \vv{w} = \vv{0}$,
where $A$ which is a matrix of dimension $d \times (d+1)$ whose rows are
the embeddings $\phi_{d+1}(\vv x)$ of the points.
This is a homogeneous linear system and so the $w$'s are unique up to a scaling factor.
Every point on the hyperplane is a linear combination of those $d$ linearly independent points
and thus also satisfies the equation
$\sum_i  w_i \alpha_i = 0$. 
Likewise, every set of weights represents at most one hyperplane crossing the system.
\end{proof}

	
\begin{figure}[h]
\centering
\includegraphics[width=0.4\linewidth]{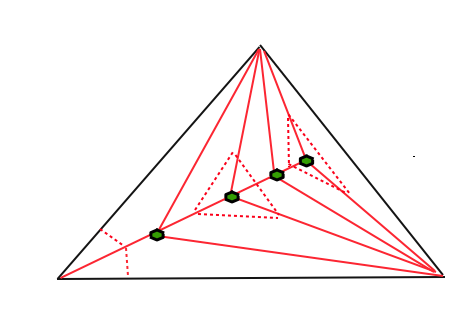}
\caption{An example of an open polytope and two disjoint polytopes as a nested barycentric system}
\label{fig:split}
\end{figure}

In Section \ref{sec:approx}, we will show that a simple nested barycentric system,
together with a prudent choice of weights, can be used to closely approximate any
given convex body. To this end, we will require a useful property of these systems
--- essentially, that splitting a simplex cannot decrease the expressiveness of the system.
Recall that a barycentric system $B_k$ is defined by an ordered set of points;
we will say that $B_k$ is {\em contained} in $B_t$ ($B_k  \subset B_t$) 
if $B_k$ is a prefix of $B_t$.

\begin{theorem}\label{thm:prefix}
\label{th:eq}
Let $B_k$ be a nested barycentric system, with $\{q_0, \ldots, q_k\}$ as its coordinates, 
and let $B_t$ be a nested barycentric system such that $B_k \subset B_t$.
Let $P$ be a polytope described in $B_k$ as: 
\begin{equation}
P= \{\vec x\in S : \vec w\cdot \phi_k(\vec x) \ge 0\}
\end{equation}
then there exist a set of weights $\vec w'$ such that $P$ can also be described by $B_t$.
Further, $\vec w$ is a prefix of $\vec w'$.
 \end{theorem}

In order to prove the theorem we must first demonstrate the relationship between
coefficients before and after a simplex split.

\begin{definition}
 Let $B_t$ be a nested barycentric system, with $\{q_0, \ldots, q_t\}$, 
 Let the new splitting point be $q_{t+1}$. Since $q_{t+1} \in B_t$, it can be written as:
\begin{equation}
 \label{step_t}
  q_{t+1} = \sum_{i=1}^{t} \beta_{i}\; \vv{q}_{i}
\end{equation}
We define $\beta_{i}$ to be the coefficients of the new coordinate of step $t+1$
using the coordinate system at step $t$.
\end{definition}

\begin{lemma}
\label{lem:1}
For a given data point $\vv p$ the connection between the coefficients of step $t$ and $t+1$ is:
\begin{equation}
 \alpha_{i,t}(\vv p)=\alpha_{i,t+1}(\vv p)+\beta_{i} \alpha_{t+1,t+1}(\vv p)
\end{equation}
\end{lemma}

For the simplicity of the notation in this proof we will use $\alpha_i$ instead of $\alpha_i(\vv p)$ , and  $\alpha^*$ instead of $\alpha_{t+1,t+1}(\vv p)$.
\begin{proof}
 The data point $\vv p$ at step $t$ can be written as:
 \begin{equation}\label{eq:exp1}
 \vv p =\sum_{i=1}^{t} {\alpha}_{i,t} \vv{q}_{i} 
\end{equation} 

The point $\vv p$ at step $t+1$ is written as:
 \begin{equation}
 \label{step_t+1}
 \vv p =\sum_{i=1}^{t+1} {\alpha}_{i,t+1} \vv{q}_{i} = \sum_{i=0}^t \alpha_{i,t+1} \vv{q}_{i} + \alpha^{*} \vv{q}_{t+1}
\end{equation}

Combining equations \ref{step_t} and \ref{step_t+1}, we derive:
 
 \begin{equation}\label{eq:exp2}
 \vv p =\sum_{i=1}^{t} (\alpha_{i,t+1}+ \beta_{i} \alpha^{*}) \vv{q}_{i}
\end{equation}

Since the barycentric representation is unique,
equations \ref{eq:exp1} and \ref{eq:exp2} together imply:
\begin{equation}
 \alpha_{i,t}=\alpha_{i,t+1}+\beta_{i} \alpha^{*}
\end{equation}
\end{proof}

We can now prove theorem \ref{th:eq} by induction:

\begin{proof}
For a given polytope $P$ with a set of weights $\vec w_t$  at system $B_t$,  such that $P= \{\vec x\in S : \vec w_t\cdot \phi_t(\vec x) \ge 0\}$,
and $B_{t} \subset B_{t+1}$, 
we choose the set of weights $\vec w_{t+1}$ for $B_{t+1}$ as follows:
The first $t$ weights of $B_{t+1}$ are the same as for $B_{t}$
($w_{i,t} = w_{i,t+1} \quad \forall i < t+1$), 
and $w_{t+1,t+1} = \sum_i \beta_i w_i(t)$.
Then the scaled distance of every given point represented in $B_t$ from the hyperplane $\gamma_t = \vec w_t \phi_t(\vec x)$, is the same as 
the scaled distance of the point represented in $B_{t+1} :  \gamma_{t+1}= \vec w_{t+1} \phi_{t+1}(\vec x)$,
and specifically the polytope $P$ remains the same.
Using Lemma \ref{lem:1} we have:
\begin{equation*}
\begin{aligned}
\gamma_{t+1}
&=  \sum_{i=0}^{t+1} \alpha_{i,t+1} w_{i,t+1}	\\
&=  \sum_{i=0}^t \alpha_{i,t+1} w_{i,t+1}+ \alpha^* w_{t+1,t+1}	\\
&=  \sum_{i=0}^t (\alpha_{i,t} - \beta_i \alpha^*) w_{i,t+1}+\alpha^* w_{t+1,t+1}	\\
&=  \underbrace{\sum_{i=0}^t \alpha_{i,t} w_{i,t}}_{ = \gamma_t} -  \alpha^* {\sum_{i=0}^t  \beta_i w_{i,t}} +{\alpha^* w_{t+1,t+1}}	\\
&=  \gamma_t +  \alpha^* \underbrace{(w_{t+1,t+1} - \sum_{i=0}^t  \beta_i w_{i,t})}_{= 0}	\\
&=  \gamma_t.
\end{aligned}
\end{equation*}
\end{proof}

\section{Approximating a convex body}\label{sec:approx}

In this section we show that the nested barycentric coordinate system (NBCS)
can represent
an arbitrarily close approximation to any convex body. 
As stated the NBCS produces a (not necessarily convex) piece-wise linear classifier.
In fact, this method can approximate multiple convex bodies.
For simplicity, we focus on the case of a single convex body, 
and demonstrate how our method approximates it.
This will be done by  placing split points at the barycenters of 
their containing simplices, where the barycenter of a simplex with vertices 
$p_0, \ldots, p_{d}$ is given by $(p_0 + \cdots + p_d)/(d+1)$.

In order to state our result formally, we introduce some notation: 
Given a point $p\in \R^d$ and a parameter $\eps>0$, 
let $B_\eps(p) = \{q\in\R^d: \|q-p\|_2\le \eps\}$ be the ball of radius $\eps$ centered at $p$. 
Given a set $X\subseteq \R^d$, let $X^{(-\eps)} = \{p\in X:B_\eps(p)\subseteq X\}$ 
be the set of all points of $X$ that are at distance at least $\eps$ from the boundary of $X$. 
Recall that $S$ denotes the unit simplex.

\begin{theorem}\label{thm:approx}
Let $P\subseteq S$ be a given convex body of diameter 1, and let $0<\eps<1$ be given. 
Then there exists a nested system $B_t$, obtained by always placing split points at the 
barycenters of their containing simplices, and a corresponding set of weights $\vec w$, 
such that
\begin{equation}\label{eq_R}
\tilde P= \{\vec x\in S : \vec w\cdot \phi_t(\vec x) \ge 0\}
\end{equation}
satisfies the following:
\begin{enumerate}
\item $\vol(\tilde P\setminus P) < \eps\vol(S)$.
\item $\tilde P^{(-\eps)} \subseteq P \subseteq \tilde P$.
\end{enumerate}
\end{theorem}

\paragraph{Proof}
The construction proceeds in stages $i=0, 1, \ldots, s$.
(Below, we will take $s = 2^{O(d)}\ln^2(1/\eps)$.) 
At stage $0$ the only points present are the vertices of $S$. 
At each stage $i$, $i\ge 1$, a new split point is placed at the barycenter of each existing simplex,
and the final construction is called the $s$-stage \emph{uniform subdivision} of $S$.
Let $A_i$ be the set of simplices present at stage $i$,
and clearly $|A_i| = (d+1)^i$. 
Note that all simplices in $A_i$ have the same volume.

The weights $w_i$ are assigned as follows: 
Initially, vertices $q_0, \ldots, q_d$ of $S$ are assigned weights $w_0 = \cdots = w_d = -1$. 
At each stage $i\ge 1$, each new split point is given the smallest possible 
weight that ensures $\tilde P\supseteq P$, where $\tilde P$ is given by (\ref{eq_R}). 
Once a weight is assigned to a point, it is never changed again. 
In other words, for those points of $B_{i+1}$ that already belonged to $B_i$, their weights at $B_{i+1}$ are the same as their weights at $B_i$.

Let $S'\in A_i$ be a simplex with vertices $q_{i_0}, \ldots, q_{i_d}$ and weights $w_{i_0}, \ldots, w_{i_d}$, respectively. 
Let $q' = (q_{i_0} + \cdots + q_{i_d})/(d+1)$ be the barycenter of $S'$. 
By Theorem \ref{th:eq}, if $q'$ is assigned weight $w_{\mathrm{avg}} = (w_{i_0} + \cdots w_{i_d})/(d+1)$, 
then $\tilde P\cap S'$ remains unchanged. 
Hence, the weight $w'$ that will be assigned to $q'$ by our construction will satisfy $w'\le w_{\mathrm{avg}}$. 
And therefore, at each stage, $\tilde P$ only shrinks.
If at stage $i$ a certain simplex $S'\in A_i$ satisfies $S'\cap P =\emptyset$, 
then at stage $i+1$ the barycenter of $S'$ will be assigned weight $-\infty$, 
so that the interior of $S'$ will lie completely outside of $\tilde P$.

Let us denote by $\tilde P_s$ the region $\tilde P$ produced by this construction after stage $s$. 
(See Figure~\ref{fig_uniform} for an illustration in the plane.)
We will now prove that, if $s$ is made large enough, then $\tilde P_s$ approximates the given convex body $P$ arbitrarily well, 
as stated in the theorem.

\begin{figure}
\centerline{\includegraphics[width=\textwidth]{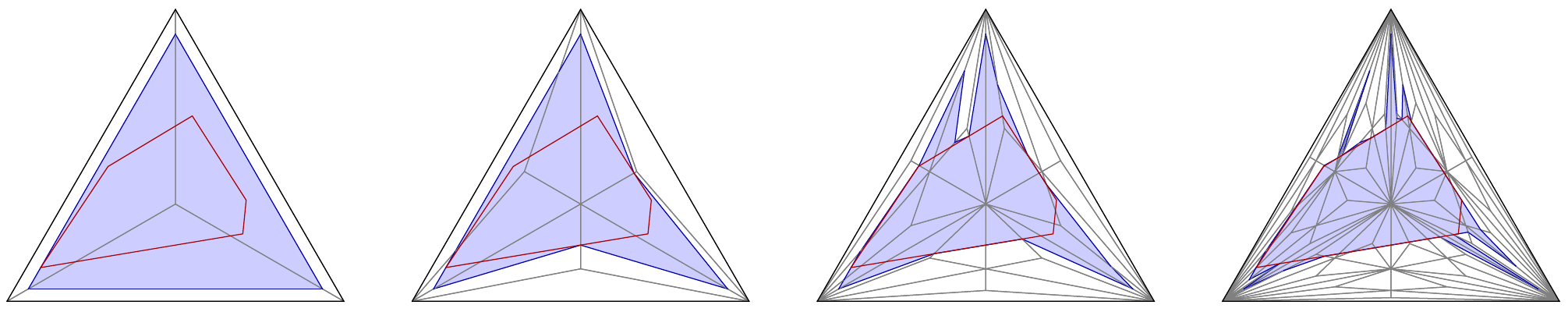}}
\caption{\label{fig_uniform}Four stages of the approximation of a given convex polygon in the plane.}
\end{figure}

The \emph{diameter} of a compact subset of $\R^d$ is the maximum distance between two points in the set. 
In particular, the diameter of a simplex is the largest distance between two vertices of the simplex.

\begin{lemma}\label{lemma_ctr_dist}
Let $S'$ be a simplex with vertices $p_0, \ldots, p_d$, let $c$ be the diameter of $S'$, and let $q$ be the barycenter of $S'$. 
Then the distance between $q$ and any vertex $p_i$ is at most $cd/(d+1)$.
\end{lemma}

\begin{proof}
Fix $p_i = \vec 0$ for concreteness. Then, under the constraints $\|p_j\|_2 \le c$ for $j\neq i$, 
the distance between $q$ and $p_i$ is maximized by letting $p_j = (c, 0, \ldots 0)$ for all $j\neq i$, 
which yields the claimed distance.
\end{proof}


\begin{lemma}\label{lemma_diam}
Let $S'$ be a simplex with diameter $c$. Let $A$ be the collection of the $(d+1)^d$ simplices obtained by a $d$-stage uniform subdivision of $S'$. Then there are at least $(d+1)!$ simplices in $A$ with diameter at most $cd/(d+1)$.
\end{lemma}

\begin{proof}
By Lemma~\ref{lemma_ctr_dist}, every simplex in $A$ that contains at most one vertex of $S'$ will have diameter at most $cd/(d+1)$. Each time a simplex $S''$ is subdivided into $d+1$ simplices by an interior point $q$, the new simplices share only $d$ of their vertices with $S''$. Hence, at stage $1$ of the subdivision of $S'$, there are $d+1$ simplices that share only $d$ vertices with $S'$; at stage $2$, there are $(d+1)d$ simplices that share only $d-1$ vertices with $S'$; and so on, until at stage $d$ there are $(d+1)d\cdots 2=(d+1)!$ simplices that share only one vertex with $S'$.
\end{proof}

Recall that $A_i$ denotes the collection of simplices present in the $i$-stage uniform subdivision of $S$.

\begin{lemma}\label{lemma_few_long}
Let $k,z$ be integers, and set $s = zkd$. Then at most a $\bigl(z(1-e^{-d})^k\bigr)$-fraction of the simplices in $A_s$ have diameter larger than $(d/(d+1))^z$.
\end{lemma}

\begin{proof}
By repeated application of Lemma~\ref{lemma_diam}. After $kd$ stages, at most an $\alpha$-fraction of the simplices in $A_{kd}$ have diameter larger than $d/(d+1)$, for $\alpha = \left(1-\frac{(d+1)!}{(d+1)^d}\right)^k$. All the other simplices have diameter at most $d/(d+1)$. Of the latter simplices, after $kd$ more stages, at most an $\alpha$-fraction of their descendants have diameter larger than $(d/(d+1))^2$. Hence, in $A_{2kd}$, the fraction of simplices with diameter larger than $(d/(d+1))^2$ is at most $\alpha + (1-\alpha)\alpha < 2\alpha$. And so on. In $A_{zkd}$, the fraction of simplices with diameter larger than $(d/(d+1))^z$ is at most $z\alpha$. Since $(d+1)!/(d+1)^d > e^{-d}$ for all $d$, the claim follows.
\end{proof}

Now, given $\eps$, let $\rho = \eps / (2\sqrt{2}d^2)$. Choose $z$ minimally so that $(d/(d+1))^z \le \rho$,
 and then choose $k$ minimally so that $z(1-e^{-d})^k \le \eps /2$. Let $s=zkd$. (Hence, we have $s\le c^d\ln^2(1/\eps)$ for some $c$.)
Let $Z_1$ be the region surrounding $P$ that is at distance at most $\rho$ from $P$, 
and let $Z_2$ be the union of all the simplices in $A_s$ with diameter larger than $\rho$. 
By the choice of $s$, every point in $\tilde P_s \setminus P$ belongs to 
$Z_1\cup Z_2$. Let us bound each of $\vol(Z_1)$ and $\vol(Z_2)$.

As $\rho\to 0$ (keeping $P$ fixed) we have $\vol(Z_1) \le (1+o(1))\rho\surf(P)$. Furthermore, $P$ and $S$ are both convex with $P\subseteq S$, so $\surf(P)\le \surf(S)$. Since $S = S_d$ where $S_d\subset\R^d$ is a regular simplex of unit side-length, we have $\vol(S_d) = \sqrt{d+1}/(d!\sqrt{2^d})$ and $\surf(S_d) = (d+1)\vol(S_{d-1}) \approx \sqrt{2}d^2\vol(S_d)$. Hence, by the choice of $\rho$, we have $\vol(Z_1) \le (\eps/2)\vol(S)$.
By Lemma~\ref{lemma_few_long}, we also have $\vol(Z_2) \le (\eps/2)\vol(S)$. 
Hence, $\vol(\tilde P_s\setminus P)\le \eps\vol(S)$, 
and the first item follows.

For the second item, by construction $P \subseteq \tilde P$.
Now given a parameter $\delta>0$, apply the first part of the theorem with $\eps = \vol(B_\delta)/(2\vol(S))$, where $\vol(B_\delta) \approx \delta^d 
\pi^{d/2} / (d/2)!$ is the volume of a $d$-dimensional ball of radius $\delta$. 
(A calculation shows that $\eps \ge \delta^d$, so it suffices to take $s=(c')^d\ln^2(1/\delta)$ for an appropriate constant $c'$.)
Suppose for a contradiction that there exists a point $p\in \tilde P_s^{(-\delta)}$ that is outside of $P$. Then the ball $B=B_\delta(p)$ is 
contained in $\tilde P$. But since $P$ is convex, more than half of $B$ is outside of $P$. Hence, $\vol(\tilde P_s\setminus P) > \vol(B)/2 = 
\eps\vol(S)$, contradicting the first part of the theorem. This implies that $\tilde P_s^{(-\delta)}\subseteq P$, concluding the second item and 
the proof of Theorem~\ref{thm:approx}.


\section{Learning algorithms}\label{sec:learning}

In Section \ref{sec:approx}, we demonstrated that the uniform subdivision embedding,
coupled with an appropriate choice of weights, can represent an approximation to any given convex body.
This motivates an embedding technique for a linear classifier.

For some parameter $q$ (determined by cross validation), 
our classification algorithm produces a $q$-stage uniform subdivision: 
Beginning with a single simplex covering the entire space, at each stage we add 
to the system the barycentric center of each simplex, 
thereby splitting all simplices into $d+1$ sub-simplices.
We call a set of $d+1$ simplices formed by a split {\em siblings}.
The procedure stops after $q$ stages, having produced $(d+1)^q$ simplices. 
We note that there is nothing to be gained by splitting an empty simplex, so the
algorithm may ignore these; then an empty simplex must have a sibling  
that contains points, and since a simplex has $d$ siblings, we have
that the total number of simplices is not greater than $\min \{(d+1)^q, dnq \}$.
Parameter $q$ is analogous to depth parameter $s$ of Lemma \ref{lemma_few_long};
however, we have consistently observed by empirical cross-validation
that it suffices to take $q$ as a very small constant (at most 5), and so
we stipulate in our algorithm that $q$ be bound by a small universal constant.


Having computed the nested coordinate system, we use it
to embed all points into high-dimensional space.
To find an appropriate weight assignment $\vec w$ for the simplex points, 
we compute a linear classifier on the embedded space to separate the data.
A linear classifier takes the form 
$h(\vec x) = \sign ( \vec w \cdot \vec x)$,
and this $\vec w$ serves as our weight vector for the embedding.
We use soft SVM as our linear classifier, and note that
the training phase
can be executed 
in time $O(dn)$ on $(d+1)$-sparse vectors \citep{joachims2006training}.
The total runtime of the algorithm is bounded by the cost of executing the sparse SVM
plus the total number of simplex points, that is
$O(\min \{d(d+1)^q + dn, d^2qn \}
= \min \{ d^{O(1)} + dn, O(d^2n) \}$).

To classify a new point, we can simply search top-bottom for its lowest containing 
simplex: We begin at the initial simplex, investigate which of its $d$ sub-simplices
contains the query point, and iterate on that simplex. This can all be done in time 
$O(qd^2) = O(d^2)$.
After bounding the run time, we want to bound the out of sample error:

\begin{theorem}\label{thm:bounds}
  If our classifier achieves sample error $\hat R$ with margin $\gamma$
  (i.e., $\hat R$ is the fraction of the points whose margin is less than $\gamma$)
on a sample of size $n$ after stopping at stage $q$, its generalization error $R$
is bounded by
\begin{equation}
  \hat R+O(1/(\gamma\sqrt n)+\sqrt{\log(q/\delta)/n})
\end{equation}  
with probability at least $1-\delta$.
\end{theorem}

This bound is a consequence of the SVM margin bound \citep[Theorem 4.5]{mohri-book2012}
and the stratification technique \citep{DBLP:journals/tit/Shawe-TaylorBWA98},
where the $q$-th stage receives weight $1/2^q$.

\paragraph{Adaptive splitting strategies.}
The above algorithm is data-independent in its selection of split points.
It is reasonable to suggest that a data-dependent
choice of split points can improve the performance of the learning algorithm.
Several greedy strategies suggest themselves, but after empirical trials
we suggest the following split heuristic: At every stage, a linear classifier of the 
embedding space is computed.
For each simplex, we identify the points in the simplex have been misclassified so far,
and choose a data point which is closest to the 
the barycentric center
{\em of the misclassified points}.
 As before, we limit the heuristic to a
constant number of stages, and it is also not necessary to subdivide an empty simplex,
or one that contains not many misclassified points.
(See Section \ref{sec:exp} for empirical results.)
The following bounds follow from Corollary \ref{cor:margin}:

\begin{theorem}\label{thm:adapt}
  If our adaptive classifier achieves sample error $\hat R$ with margin $\gamma$
  (i.e., $\hat R$ is the fraction of the points whose margin is less than $\gamma$)
on a sample of size $n$ after stopping at stage $q$ and retaining $k$ split points, its generalization error $R$
is bounded by
\begin{equation}
  \hat R+O(1/(\gamma\sqrt{n-k})+\sqrt{\log(q/\delta)/(n-k)})
\end{equation}  
with probability at least $1-\delta$.
\end{theorem}

\section{Experiments}\label{sec:exp}
Our embedding technique is {\em motivated} by provable bounds for convex polytopes,
but we find that it is sufficiently robust to yield impressive empirical results 
for non-convex polytopes or even general point sets.
All experiments utilized the python scikit-learn library \citep{scikit-learn} \footnote{code can be found at https://github.com/erankfmn/NBCS-embedding}.
The regularization parameter $C$ was 5-fold cross-validated 
over the set $\{2^{-5}, 2^{-3}, \ldots, 2^{15}\}$,
and for the RBF kernel, the $\gamma$
parameter was five-fold cross-validated over the set $\{2^{-15}, 2^{-3},\ldots, 2^{3}\}$. 
For our methods, the maximum iteration parameter $q$ was cross validated over the set $\{2 , \ldots ,5\}$.
Our algorithms usually converged even before reaching the maximum number of allowed iterations.

\begin{figure}
\centering
\subfloat[step 1]{\includegraphics[width=0.5\linewidth]{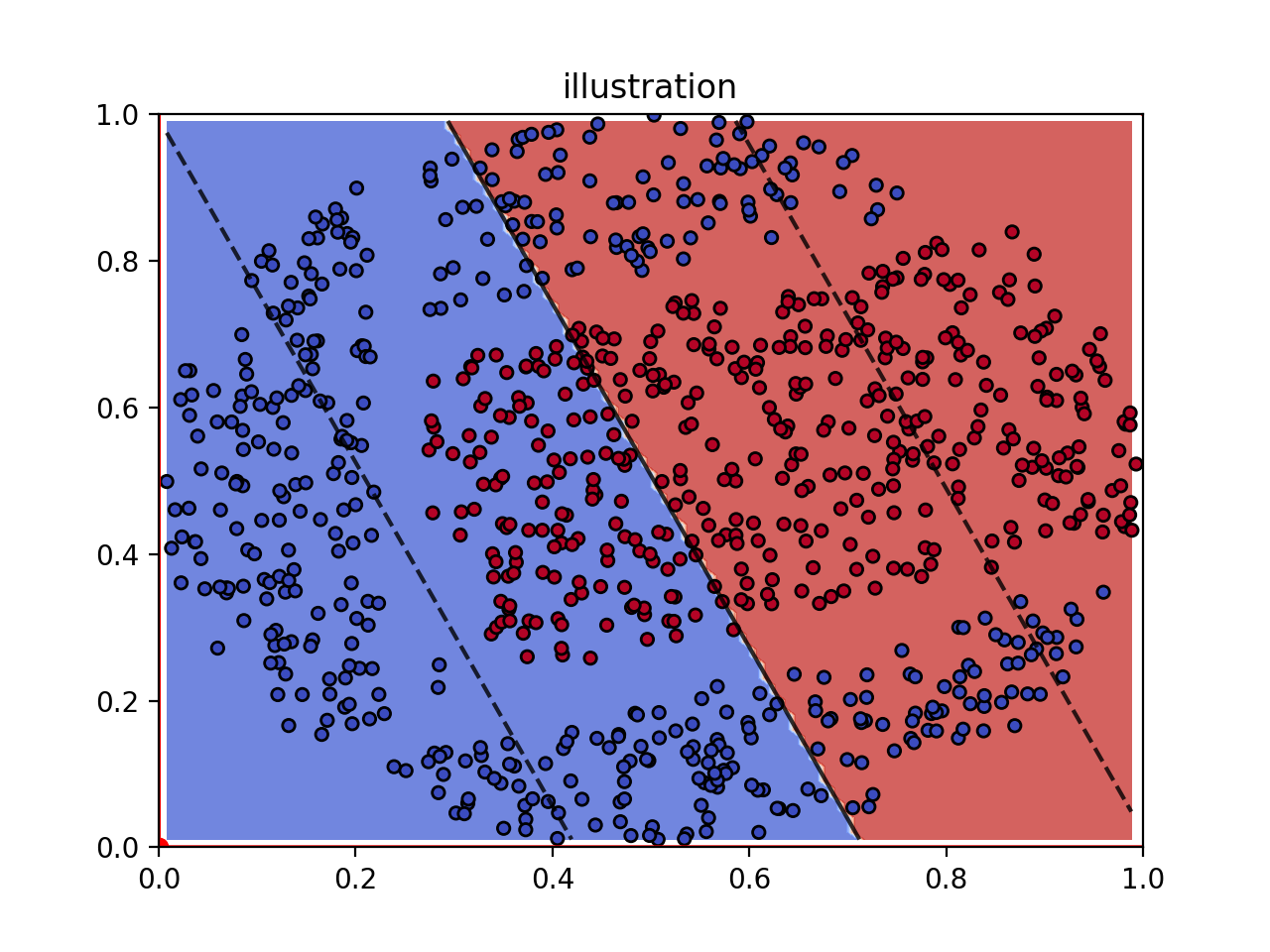}}
\subfloat[step 2]{\includegraphics[width=0.5\linewidth]{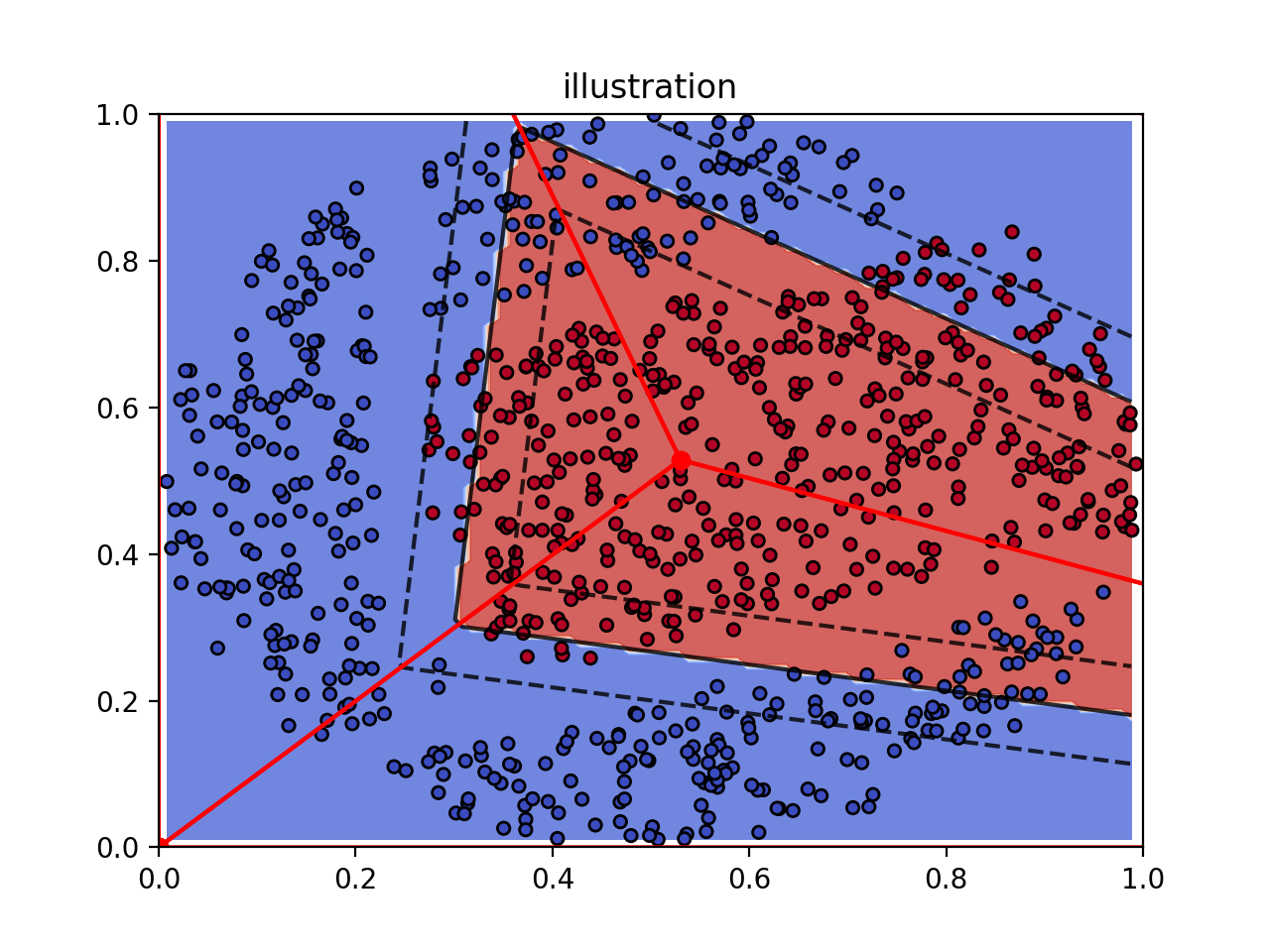}}\\
\subfloat[step 3]{\includegraphics[width=0.5\linewidth]{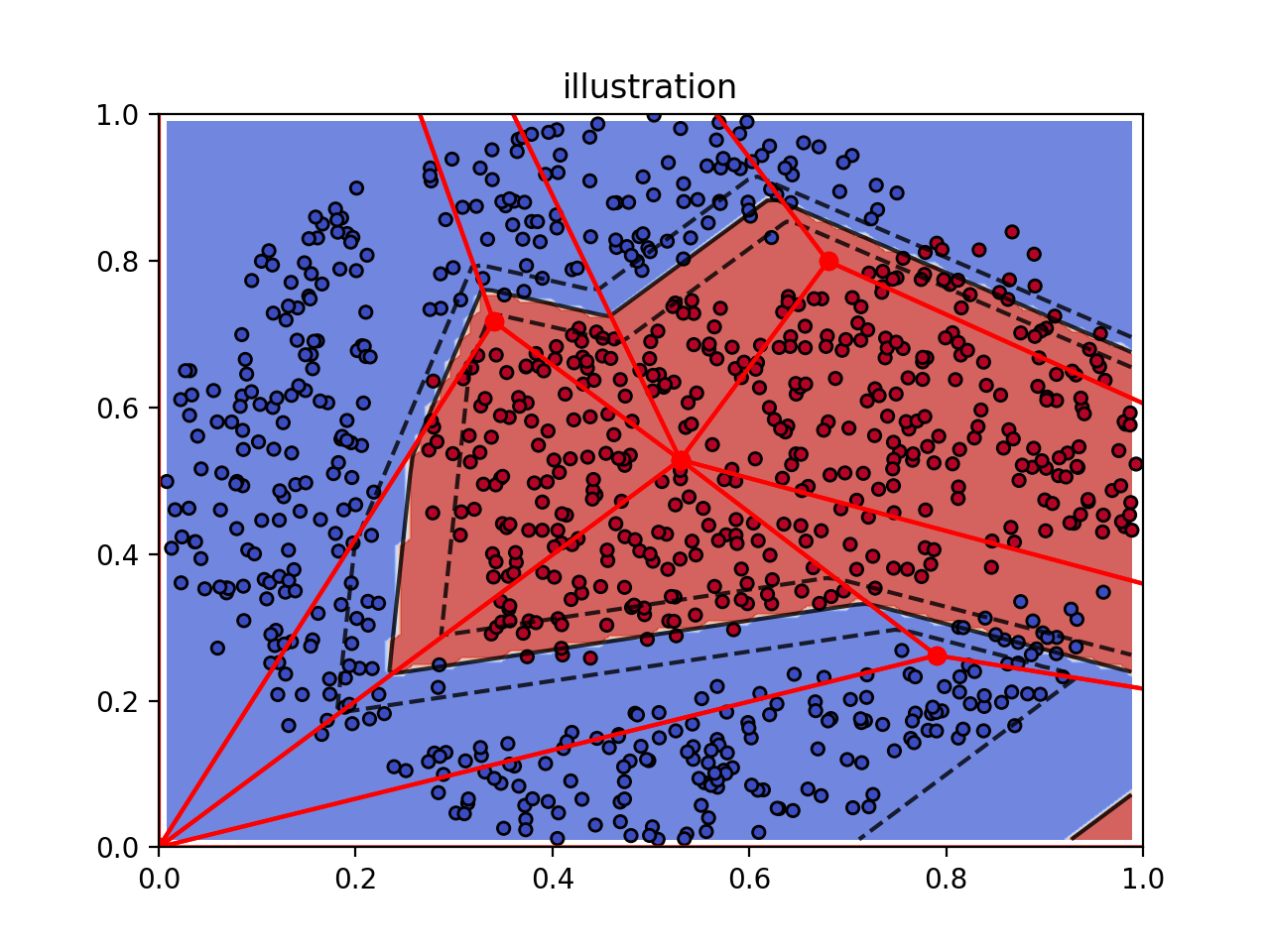}}
\subfloat[step 4]{\includegraphics[width=0.5\linewidth]{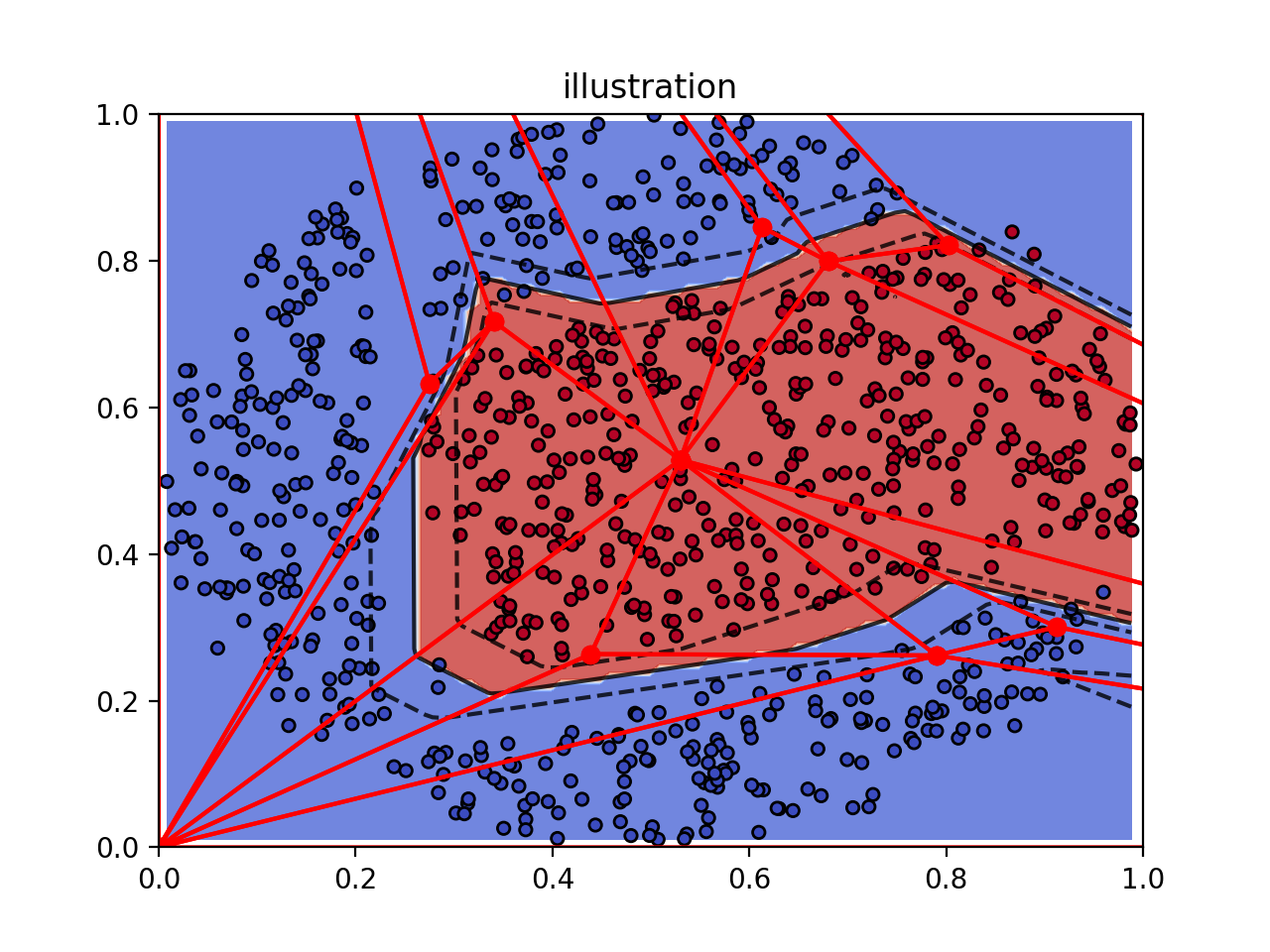}}
\caption{Learning a polytope separating the red and blue points.}
\label{fig:poly2}
\end{figure}

\paragraph{Non-convex polytope approximation.}
Before presenting the experiments, we give a simple example that illustrates the power of our approach
in approximating non-convex polytopes. We created a random data-set wherein all positive examples were taken 
from within a $5$-gon and the negative points from outside it.
This data was randomly generated within the unit circle: 
Each vector was sampled from the unit sphere and then normalized by $u^{1/d}$, 
where random variable $u \in [0, 1]$ is sampled independently at random for each vector.
We then sampled 5 halfspaces whose intersection formed the target polytope: 
For each halfspace, we sampled a random direction vector $w_j$ uniformly from the unit sphere,
and then sampled a random offset value $b_j \in [.05, .95]$ to produce the halfspace $(w_j,b_j)$. 
The intersection of these halfspaces is the target polytope. 
All data points inside the polytope with margin 0.05 were labeled as positive, all data points outside 
the polytope with margin 0.05 were labeled as negative, and the rest were discarded.

Figure \ref{fig:poly2} shows the iterative boundary formation,
where the bold black line is the decision boundary and the dotted lines are the 
margin ($\vec w \cdot \phi_t(x) = \pm 1$) .
For each iteration, the nested barycentric system is illustrated by the red lines.
A consistent approximation of the underlying polytope for multiple runs
was achieved after only 3 iterations.
Notice how the margins become smaller at each iteration until reaching their predetermined size.

\paragraph{Benchmarks.}
We first compared the runtime and accuracy of our methods in Section \ref{sec:learning}
-- uniform subdivision with NBCS (uni-NBCS) and adaptive splitting with NBCS (adapt-NBCS) --
to the 
2\textsuperscript{nd}
and 
3\textsuperscript{rd}
degree polynomial explicit feature maps, 
and to the RBF kernel SVM.
We used CoreSVM for the RBF kernel, as the LibSVM RBF failed to run on very large datasets.
We considered large datasets from LibSVM Machine Learning repository \citep{CC01a},
taking random $70-30\%$ splits averaged over $10$ random trials.
In the LibSVM implementation, 
the runtime of 2\textsuperscript{nd} degree SVM is $O(d^2 n)$
and 3\textsuperscript{rd} degree SVM is $O(d^3n)$. 
We implemented our algorithm to run in $O(d^2 n)$ time.
Table \ref{tab:res} shows a summary of our experimental results,
and demonstrates that our method compares favorably to the others
both in runtime and accuracy.
We believe that this is due to NBCS embedding the data into a small but
yet very expressive space.
We further compared our technique to other explicit feature map methods. 
Here we focused on accuracy as opposed to runtime, 
since all these methods have similar runtime complexity,
Figure \ref{fig:comp} demonstrates a comparison of the average accuracy between 
our embedding technique (adapt-NBCS), Kitchen Sink (KS) \citep{rahimi2007random}, 
Nystrom's approximation \citep{Williams00theeffect} 
and the adaptive $\chi^2$ \citep{vedaldi12efficient}, 
all of which have open source implementations,
over a large variety of medium sized datasets.
We also included the accuracy achieved by RBF CoreSVM.
Again, our algorithm's accuracy compared favorably with the others.

\section{Discussion and future work}
In this paper, we introduced the barycentric coordinate system embedding,
demonstrated its computational power, and suggested implementation techniques.
We derived a statistical foundation for this approach, and presented
experiments on LibSVM datasets which show promising empirical results.
This method is advantageous in the large data and small to medium feature size regime. 
Future work includes analytical and empirical investigations of other 
natural splitting strategies.

\begin{table*} 
\begin{center} 
\begin{tabular}{lccccccc}
\toprule
{\small Dataset} & {$n$} & {$d$} & {\small {2nd degree SVM}} & {\small {3rd degree SVM}} & {\small {CoreSVM-RBF}} & \textbf{\small {uni-NBCS}} & \textbf{\small {adapt-NBCS}}\\
\midrule
{\small letter} & 20,000 & 16 & {\small 84\%, 12.3 sec} & {\small 89.2\%, 81.7 sec} & {\small 81\%, 38 sec} & \textbf{\small 90.5\%, 14 sec} & \textbf{\small {91.5\%, 17 sec}}\\
{\small SkinNonSkin} & 245,057 & 4 & {\small 99.25\%, 13 sec} & {\small 99.4\%, 20 sec} & {\small 98.9\%, 570.4 sec} & \textbf{\small 97.6\%, 4 sec} & \textbf{\small 98.8\%, 4.2 sec}\\
{\small cod-rna} & 59,535 & 8 & {\small 94.9\%, 9 sec} & {\small 95.2\%, 18.6 sec} & {\small 94.3\%, 23 sec} & \textbf{\small 93.6\%, 8.7 sec} & \textbf{\small 94.5\%, 9 sec}\\
{\small shuttle} & 58,000 & 9 & {\small 96\%, 8 sec} & {\small 98\%, 25.8 sec} & {\small 93.2\%, 5.3 sec} & \textbf{\small 95.4\%,  8 sec} & \textbf{\small 97.8\%, 6.3 sec }\\
{\small covtype} & 581,012 & 54 & {\small 79\%, 1950 sec} & {\small 81.5\%, 8028 sec} & {\small 83.5\%, 10028 sec}  & \textbf{\small 82.3\%, 2040 sec} & \textbf{\small 82.5\%, 2140 sec}
\end{tabular} 
\end{center} 
\caption{UCI Datasets} \label{tab:res} 
\end{table*}

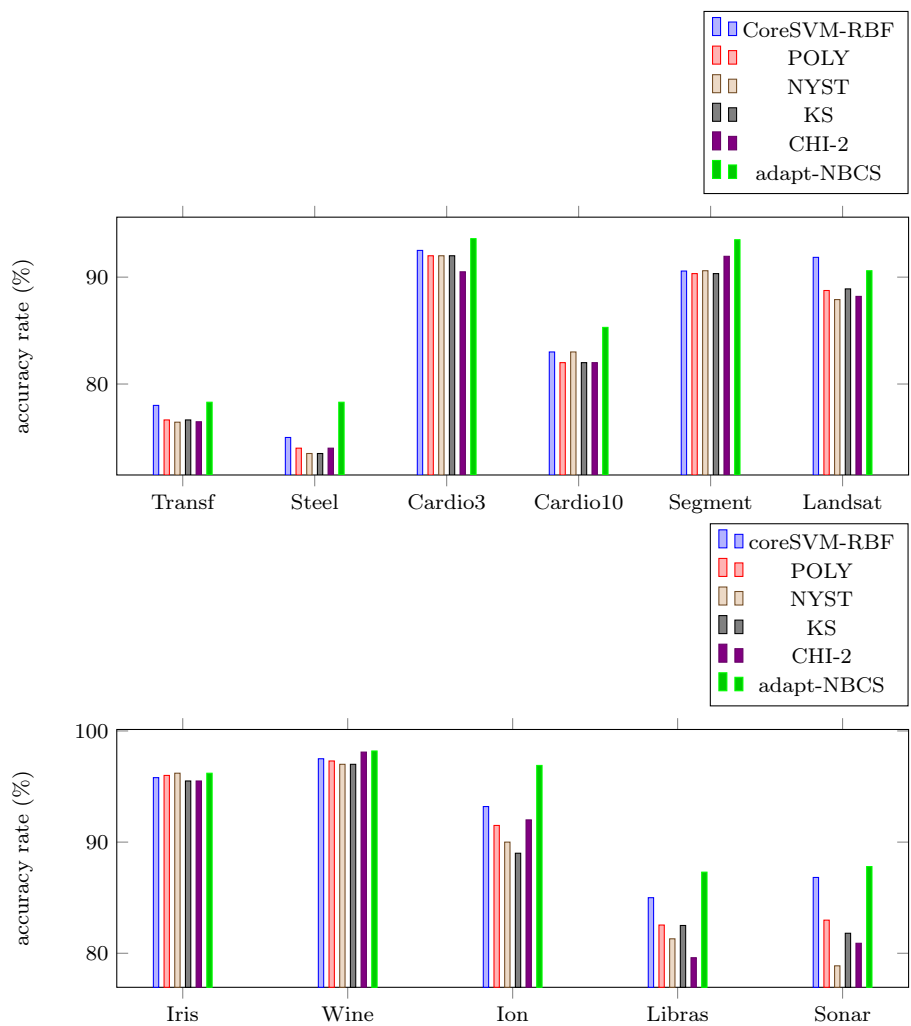
\begin{figure*}
\centering
\begin{tikzpicture}
    \begin{axis}[
        ybar,
        width  = 12cm,
        height = 5cm,
        bar width=2pt,
	legend style={at={(1,1.8)}},
        ylabel={accuracy rate (\%)},
        symbolic x coords={ Transf , Steel , Cardio3 , Cardio10 , Segment , Landsat ,},
        xtick = data,
    ]

\addplot coordinates { (Transf , 78) (Steel , 75) (Cardio3 , 92.5) (Cardio10 , 83)(Landsat , 91.85) (Segment , 90.57) }; 

\addplot coordinates { (Transf , 76.64) (Steel , 74) (Cardio3 , 92) (Cardio10 , 82) (Landsat , 88.75) (Segment , 90.33) }; 

\addplot coordinates { (Transf , 76.43) (Steel , 73.5)  (Cardio3 , 92) (Cardio10 , 83) (Landsat , 87.9) (Segment , 90.6) };

\addplot coordinates { (Transf , 76.64) (Steel , 73.5)  (Cardio3 , 92) (Cardio10 , 82) (Landsat , 88.9) (Segment , 90.33) };
\addplot coordinates { (Transf , 76.47) (Steel , 74)  (Cardio3 , 90.5) (Cardio10 , 82) (Landsat , 88.2) (Segment ,91.95) };

\addplot coordinates { (Transf , 78.3) (Steel , 78.3) (Cardio3 , 93.6) (Cardio10 , 85.3) (Landsat , 90.6) (Segment , 93.5)};

\legend{CoreSVM-RBF,POLY,NYST,KS,CHI-2 ,adapt-NBCS}
    \end{axis}
\end{tikzpicture}

\begin{tikzpicture}
    \begin{axis}[
        ybar,
        width  = 12cm,
        height = 5cm,
        bar width=2pt,
	legend style={at={(1,1.8)}},
        ylabel={accuracy rate (\%)},
        symbolic x coords={Iris, Wine , Ion , Libras , Sonar , },
        xtick = data,
    ]

\addplot coordinates { (Iris, 95.8) (Wine , 97.5) (Ion , 93.2) (Libras , 85) (Sonar , 86.82)    }; 

\addplot coordinates { (Iris, 96) (Wine , 97.3) (Ion , 91.5) (Libras , 82.54) (Sonar , 82.98)  }; 

\addplot coordinates { (Iris, 96.20) (Wine , 97) (Ion , 90) (Libras , 81.3) (Sonar , 78.87)  };

\addplot coordinates { (Iris, 95.5) (Wine , 97) (Ion , 89) (Libras , 82.5) (Sonar , 81.8) };
\addplot coordinates { (Iris, 95.5) (Wine , 98.1) (Ion , 92) (Libras , 79.6) (Sonar , 80.9)  };
\addplot coordinates {(Iris, 96.20) (Wine , 98.2) (Ion , 96.9) (Libras , 87.30) (Sonar , 87.8) };

\legend{coreSVM-RBF,POLY,NYST,KS,CHI-2, adapt-NBCS}
    \end{axis}
\end{tikzpicture}
\caption{classification results for different embedding techniques}
\label{fig:comp}
\end{figure*}

\newpage

\bibliographystyle{spbasic}      
\bibliography{paper}   

\appendix

\section{Hybrid PAC-compression bounds}

In this section, we present a hybrid compression bound used in the derivation
of Theorem \ref{thm:adapt}.

\paragraph{General theory.}

It will be convenient to present our results in generality and then specialize.
Our notation will be in line with \cite{DBLP:conf/alt/hk18}.
Let
$P$
be
a distribution on $\Z$.
We write $Z_{[n]}=(Z_1,\ldots,Z_n)\sim P^n$ and, for $f\in[0,1]^\Z$,
\beq
R(f,P)
:=\E_{Z\sim P}f(Z),
\qquad
\hat R(f,Z_{[n]})
:=\frac1n\sum_{i=1}^nf(Z_i).
\eeq
We write $\Delta_n(f)=
\Delta_n(f,P,Z_{[n]})
:=R(f,P)-\hat R(f,Z_{[n]})
$
and our main object of interest will be
\beqn
\label{eq:barD}
\bar\Delta_n(\F):=\sup_{f\in\F}\Delta_n(f,P,Z_{[n]}),
\eeqn
for $\F\subset[0,1]^Z$. The catch is that $\F$ may itself be random,
determined by the $Z_{[n]}$. We will distinguish
$\bar\Delta_n(\F)$ from the more familiar object
$\bar\Delta_n\fix(\F)$,
which is formally defined as in (\ref{eq:barD}),
but with the additional stipulation that $\F$ be a fixed function class,
independent of $Z_{[n]}$.

For a fixed $k\in\N$,
consider a fixed mapping $\rho:\Z^k\mapsto\F_k\subset[0,1]^Z$.
In words, $\rho$ maps $k$-tuples over $\Z$ into function classes over $\Z$.
This generalizes the notion of a {\em decoding} in a sample compression scheme,
where $\rho$ maps a $k$-tuple over $\Z$ into a {\em single} function $f\in[0,1]^\Z$.
Denote by $\F_\rho(Z_{[n]})$ the collection of all functions constructable
by $\rho$ on a given $Z_{[n]}$:
\beqn
\label{eq:Frho}
\F_\rho(Z_{[n]}) = \bigcup_{
  I\in{[n]\choose k}
}\rho(Z_I),
\eeqn
where
$[n]\choose k$ is the set of all $k$-subsets of $[n]$,
and $Z_I$ is
the restriction of
$Z_{[n]}$ to the index set $I$.
\footnote{
  We consider, for concreteness, permutation and repetition-invariant
  compression schemes; the extension to general ones is straightforward.
  The only requisite change consists of replacing
  $\cup_{I\in{[n]\choose k}}$ with
  $\cup_{I\in{[n]^k}}$ in
  (\ref{eq:Frho}).
}

A trivial application of the union bound yields
\beq
\P\paren{
  \bar\Delta_n(\F_\rho(Z_{[n]}))
  \ge\eps
  }
&\le& {n\choose k}\max_{ I\in{[n]\choose k}}
\P\paren{
\bar\Delta_n(\rho(Z_I))\ge\eps
}.
\eeq

The key observation is that, conditioned on $Z_I$,
the function class $\F_I:=\rho(Z_I)$ becomes deterministic
and independent of $Z_J$, where $J:=[n]\setminus I$.
Thus,
\beq
\P\paren{
\bar\Delta_n(\F_I)\ge\eps
}
&=&
\E_{Z_I}\sqprn{
\P\paren{
  \bar\Delta_n(\F_I)\ge\eps
  \gn Z_I
}
}
.
\eeq
Conditional on $Z_I$, we have, for a given $f\in\F_I$,
\beq
\Delta_n(f,P,Z_{[n]})
=
R(f,P)-\hat R(f,Z_{[n]})
&=&
\E_{Z\sim P}f(Z)
-
\frac1n\sum_{i=1}^nf(Z_i)
\\&=&
\E f(Z)
-
\frac1n\sum_{i\in J}f(Z_i)
-
\frac1n\sum_{i\in I}f(Z_i)
\\&\le&
\E f(Z)
-
\frac1n\sum_{i\in J}f(Z_i)
\\&\le&
\E f(Z)
-
\frac1{|J|}\sum_{i\in J}f(Z_i)
\\&=&
R(f,P)-\hat R(f,Z_{J})=
\Delta_{n-k}(f,P,Z_{J}),
\eeq
where $f(\cdot)\in[0,1]$ and $|J|=n-k$ were used.
It follows that
\beq
\bar\Delta_n(\F_I)\le
\bar\Delta_{n-k}\fix(\F_I).
\eeq


We now state the main result of this section:
\begin{theorem}
  \label{thm:hybrid}
\beqn
\label{eq:main-pac}
\P\paren{
  \bar\Delta_n(\F_\rho(Z_{[n]}))
  \ge\eps
  }
&\le&
{n\choose k}
\max_{ I\in{[n]\choose k}}        
\P\paren{
\bar\Delta_{n-k}\fix(\F_I)  
  \ge\eps
}.
\eeqn
\end{theorem}

To apply (\ref{eq:main-pac}) to examples of interest,
let us compute the right-hand side of the bound for some function classes.

\paragraph{Example: VC classes.}

In our first example, suppose that $\rho$ maps $k$-tuples of $\Z$
to binary concept classes --- which might well be different for each $k$-tuple ---
of VC-dimension at most $d$.
More precisely, we take $\Z=\X\times\set{0,1}$,
where $\X$ is an instance space. Let $\H=\H_z\subseteq\set{0,1}^\X$ be a concept
class
defined by the $k$-tuple $z\in\Z^k$,
with VC-dimension $d$. Define $\F\subseteq\set{0,1}^\Z$ to
be its associated loss class:
\beq
\F=\set{ f_h : (x,y)\mapsto\pred{h(x)\neq y} ; h\in \H}.
\eeq
We call this setting a {\em hybrid} $(k,d)$ VC sample-compression scheme.
It is well-known
(see, e.g., \cite[Theorem 4.9]{MR1741038})
that
\beqn
\label{eq:E-vc}
\E[ \bar\Delta_{n}\fix(\F)]
\le c\sqrt{{d}/{n}},
\eeqn
where $c>0$ is a universal constant (for concreteness, we may take $c=144$)\footnote{
  \url{https://www.cs.bgu.ac.il/~asml162/wiki.files/dudley-pollard.pdf}
}.
Further, $\bar\Delta_{n-k}\fix(\F)$ is known to be concentrated about its mean
(see, e.g., \cite[Theorem 3.1]{mohri-book2012}):
\beqn
\label{eq:mcd}
\P\paren{
  \bar\Delta_{n}\fix(\F)
  \ge
  \E[ \bar\Delta_{n}\fix(\F)]
  +\eps
  }\le\exp(-2n\eps^2).
\eeqn

Combining (\ref{eq:main-pac}),
(\ref{eq:E-vc}), and
(\ref{eq:mcd}), we conclude:
\begin{corollary}
  \label{cor:hybrid-vc}
In a hybrid $(k,d)$
VC sample compression scheme,
on a sample of size $n$,
a learner's sample error $\herr(\hat h_n)$
and generalization error $\err(\hat h_n)$
satisfy
\beq
\err(\hat h_n)
\le
\herr(\hat h_n)
+
c\sqrt{\frac{d}{n-k}}
+
\sqrt{
  \frac{\log[\delta\inv{n\choose k}]}{2(n-k)}
}
\eeq
with probability at least $1-\delta$.
\end{corollary}

\paragraph{Example: Margin classes.}
Here, we take $\X$ to be an abstract set,
$\Y=\set{-1,1}$, $\Z=\X\times\Y$,
and define
\beq
\tilde\H=\set{h_w:\X\ni x\mapsto w\cdot\Psi(x) ; \nrm{w}\le1},
\eeq
where $\Psi(x)=\Psi_z(x)$ is a map from $\X$ to $\R^N$
determined by some $k$-tuple $z\in\Z$, with $\nrm{\Psi_z(\cdot)}\le1$.
Associate to $\tilde\H$ the $\gamma$-margin loss class
\beq
\F_\gamma=\set{f_h:\X\times\set{-1,1}\ni(x,y)\mapsto\Phi_\gamma(yh(x)); h\in\tilde\H},
\eeq
where $\Phi_\gamma(t)=\min(0,\max(1,1-t/\gamma))$.
We refer to this setting as a {\em hybrid} $(k,\gamma)$ margin sample compression scheme.
It is a standard fact (see, e.g., \cite[Theorem 4.4]{mohri-book2012}) that
\beqn
\label{eq:marg-bd}
\P\paren{
  \bar\Delta_{n}\fix(\F_\gamma)
  \ge
  \frac{2}{\gamma\sqrt n}
  +\eps
  }\le\exp(-2n\eps^2).
\eeqn

Combining (\ref{eq:main-pac}),
(\ref{eq:marg-bd}),
and a standard stratification argument (see 
\cite[Theorem 4.5]{mohri-book2012}), we obtain the following result.
Fix a map $\rho:\Z^k\to\Psi(\cdot)$.
Given a sample $Z_{[n]}=(X_i,Y_i)_{i\in[n]}$ drawn iid,
the learner chooses some $k$ examples to define the random mapping $\Psi_z:\X\to\R^N$.
Having mapped the sample to $R^N$, he runs SVM and obtains a hyperplane $w$.
\begin{corollary}\label{cor:margin}
With probability at least $1-\delta$, we have
\beq
\E_{(X,Y)}[\sgn(Yw\cdot\Psi(X))\le0\gn Z_{[n]}]
&\le&
\frac1n\sum_{i=1}^n\max(0,1-Y_iw\cdot\Psi(X_i))
+
\frac{4}{\nrm{w}\sqrt{n-k}}
\\&+&
\sqrt{\frac{\log\log_2\frac{2}{\nrm{w}}}{n-k}}
+
\sqrt{\frac{\log(2{n\choose k}/\delta)}{2(n-k)}}.
\eeq
\end{corollary}

\end{document}